\documentclass[10pt,journal,compsoc]{IEEEtran}
\ifCLASSOPTIONcompsoc
  \usepackage[nocompress]{cite}
\else
  \usepackage{cite}
\fi
\ifCLASSINFOpdf
  \usepackage[pdftex]{graphicx}
\else
\fi
\usepackage{amsmath}
\usepackage{algorithmic}

\usepackage{xspace}
\usepackage{color}
\usepackage{amssymb,amsfonts,amsmath,amsthm}
\usepackage{changepage} 
\usepackage{algorithm}
\usepackage{booktabs}       

\newtheorem{mydef}{Definition}
\newtheorem{thm}{Theorem}
\newtheorem{lemma}[thm]{Lemma}

\newcommand\flyhash{FlyHash\xspace}
\newcommand\densefly{DenseFly\xspace}
\newcommand\simhash{SimHash\xspace}
\newcommand\wtahash{WTAHash\xspace}
\newcommand\flyhashmp{FlyHash-MP\xspace}

\newcommand\mnist{{MNIST}\xspace}
\newcommand\sift{{SIFT-1M}\xspace}
\newcommand\gist{{GIST-1M}\xspace}
\newcommand\random{{Random}\xspace}
\newcommand\glove{{GLoVE}\xspace}
\newcommand\labelme{{LabelMe}\xspace}

\newcommand{\lagr}{\mathcal{L}}
\newcommand{\bd}{\mathcal{B}}
\newcommand{\normal}{\mathcal{N}}
\newcommand{\expt}{\mathbb{E}}

\begin{document}
%
\title{Improving Similarity Search with High-dimensional Locality-sensitive Hashing}
%
%
%
%

\author{Jaiyam Sharma
        and~Saket~Navlakha
\IEEEcompsocitemizethanks{\IEEEcompsocthanksitem J.\@ Sharma is with the Integrative Biology Laboratory, Salk Institute for Biological Studies, La Jolla, CA 92037.\protect\\
E-mail: jaiyamsharma@gmail.com
\IEEEcompsocthanksitem S.\@ Navlakha is with the Integrative Biology Laboratory, Salk Institute for Biological Studies, La Jolla, CA 92037.
\protect\\
E-mail: navlakha@salk.edu}

\thanks{}}

%
%

\markboth{J. Sharma, S. Navlakha, December~2018}%
{Shell \MakeLowercase{\textit{et al.}}: Bare Demo of IEEEtran.cls for Computer Society Journals}
%



\IEEEtitleabstractindextext{%
\begin{abstract}
We propose a new class of data-independent locality-sensitive hashing (LSH) algorithms based on the fruit fly olfactory circuit. The fundamental difference of this approach is that, instead of assigning hashes as dense points in a low dimensional space, hashes are assigned in a high dimensional space, which enhances their separability. We show theoretically and empirically that this new family of hash functions is  locality-sensitive and preserves rank similarity for inputs in any $\ell_p$ space. We then analyze different variations on this strategy and show empirically that they outperform existing LSH methods for nearest-neighbors search on six benchmark datasets. Finally, we propose a multi-probe version of our algorithm that achieves higher performance for the same query time, or conversely, that maintains performance of prior approaches while taking significantly less indexing time and memory. Overall, our approach leverages the advantages of separability provided by high-dimensional spaces, while still remaining computationally efficient.
\end{abstract}

}

\maketitle

\IEEEdisplaynontitleabstractindextext

%
\IEEEpeerreviewmaketitle

\IEEEraisesectionheading{\section{Introduction}\label{sec:introduction}}

%
%
%
%


\IEEEPARstart{S}{imilarity} search is an essential problem in machine learning and information retrieval~\cite{Andoni2008,Wang2014}. The main challenge in this problem is to efficiently search a database to find the most similar items to a query item, under some measure of similarity. While efficient methods exist when searching within small databases (linear search) or when items are low dimensional (tree-based search~\cite{Samet2005,Liu2005}), exact similarity search has remained challenging for large databases with high-dimensional items~\cite{Gionis1999}. 

This has motivated \emph{approximate} similarity search algorithms, under the premise that for many applications, finding reasonably similar items is ``good enough'' as long as they can be found quickly. One popular  approximate similarity search algorithm is based on \emph{locality-sensitive hashing} (LSH~\cite{Indyk1998}). The main idea of many LSH approaches is to compute a low-dimensional hash for each input item, such that similar items in input space lie nearby in hash space for efficient retrieval.


In this work, we take inspiration from the fruit fly olfactory circuit~\cite{Dasgupta2017} and consider a conceptually different hashing strategy, where instead of hashing items to lie densely in a low-dimensional space, items are hashed to lie in a high-dimensional space. We make the following contributions: 
\begin{enumerate}
\item We present and analyze the locality-sensitive properties of a binary high-dimensional hashing scheme, called \densefly. We also show that \densefly outperforms previous algorithms on six benchmark datasets for nearest-neighbors retrieval;
\item We prove that \flyhash (from Dasgupta et al.~\cite{Dasgupta2017}) preserves rank similarity under any $\ell_p$ norm, and empirically outperforms prior LSH schemes designed to preserve rank similarity; 
\item We develop and analyze multi-probe versions of \densefly and \flyhash that improve search performance while using similar space and time complexity as prior approaches. These multi-probe approaches offer a solution to an open problem stated by Dasgupta et al.~\cite{Dasgupta2017} of learning efficient binning strategies for high-dimensional hashes, which is critical for making high-dimensional hashing usable for practical applications.
\end{enumerate}

\section{Related work}


Dasgupta et al.~\cite{Dasgupta2017} recently argued that the fruit fly olfactory circuit uses a variant of locality-sensitive hashing to associate similar behaviors with similar odors it experiences. When presented with an odor (a ``query''), fruit flies must recall similar odors previously experienced in order to determine the most appropriate behavioral response. Such a similarity search helps to overcome noise and to generalize behaviors across similar stimuli.

The fly olfactory circuit uses a three-layer architecture to assign an input odor a hash, which corresponds to a set of neurons that fire whenever the input odor is experienced. The first layer consists of 50 odorant receptor neuron (ORN) types, each of which fires at a different rate for a given odor. An odor is thus initially encoded as a point in 50-dimensional space, $\mathbb{R}_{+}^{50}$. The distribution of firing rates for these 50 ORNs approximates an exponential distribution, whose mean depends on the odor's concentration~\cite{Hallem2006,Stevens2016}. The second layer consists of 50 projection neurons (PNs) that receive feed-forward input from the ORNs and recurrent inhibition from lateral inhibitory neurons. As a result, the distribution of firing rates for the 50 PNs is mean-centered, such that the mean firing rate of PNs is roughly the same for all odors and odor concentrations~\cite{Root2008,Asahina2009,Olsen2010,Stevens2015}. The third layer \emph{expands} the dimensionality: the 50 PNs connect to 2000 Kenyon cells (KCs) via a sparse, binary random projection matrix~\cite{Caron2013}. Each KC samples from roughly 6 random PNs, and sums up their firing rates. Each KC then provides feed-forward excitation to an inhibitory neuron call APL, which provides proportional feed-back inhibition to each KC. As a result of this winner-take-all computation, only the top $\sim$5\% of the highest firing KCs remain firing for the odor; the remaining 95\% are silenced. This 5\% corresponds to the hash of the odor. Thus, an odor is assigned a sparse point in a higher dimensionality space than it was originally encoded (from 50-dimensional to 2000-dimensional).

The fly evolved a unique combination of computational ingredients that have been explored piece-meal in other LSH hashing schemes. For example, \simhash~\cite{Charikar2002,Datar2004} and Super-bit LSH~\cite{Ji2012} use dense Gaussian random projections, and FastLSH~\cite{Dasgupta2011} uses sparse Gaussian projections, both in low dimensions to compute hashes, whereas the fly uses sparse, binary random projections in high dimensions. Concomitant LSH~\cite{Eshghi2008} uses high dimensionality, but admit to exponentially increasing computational costs as the dimensionality increases, whereas the complexity of our approach scales linearly. WTAHash~\cite{Yagnik2011} uses a local winner-take-all mechanism (see below), whereas the fly uses a global winner-take-all mechanism. Sparse projections have also been used within other hashing schemes~\cite{KN14,AllenZhu2014,Shi2009,Li2006,Andoni2015}, but with some differences. For example, Li et al.~\cite{Li2006} and Achlioptas~\cite{Achlioptas2003} use signed binary random projections instead of 0-1 random projections used here; Li et al.~\cite{Li2006} also do not propose to expand the hash dimensionality.



\section{Methods}

We consider two types of binary hashing schemes consisting of hash functions, $h_1$ and $h_2$. The LSH function $h_1$ provides a distance-preserving embedding of items in $d$-dimensional input space to $mk$-dimensional binary hash space, where the values of $m$ and $k$ are algorithm specific and selected to make the space or time complexity of all algorithms comparable (Section~\ref{sec:eval}). The function $h_2$ places each input item into a discrete bin for lookup. Formally: 
%
\begin{adjustwidth}{3mm}{3mm}
\begin{mydef}
A hash function $h_1: \mathbb{R}^d \rightarrow \{0,1\}^{mk}$ is called locality-sensitive if for any two input items $p,q \in \mathbb{R}^d$, $\mathrm{Pr}[h_1(p) = h_1(q)] = \mathrm{sim}(p,q)$, where $\mathrm{sim}(p,q) \in [0,1]$ is a similarity measure between $p$ and $q$.
\end{mydef}
\end{adjustwidth}
%
\begin{adjustwidth}{3mm}{3mm}
\begin{mydef}
A hash function $h_2: \mathbb{R}^{d} \rightarrow [0, \dots, b]$ places each input item into a discrete bin.\\
\end{mydef}
\end{adjustwidth}

\noindent Two disadvantages of using $h_1$ --- be it low or high dimensional --- for lookup are that some bins may be empty, and that true nearest-neighbors may lie in a nearby bin. This has motivated \emph{multi-probe LSH}~\cite{Lv2007} where, instead of probing only the bin the query falls in, nearby bins are searched, as well. 

Below we describe three existing methods for designing $h_1$ (\simhash, \wtahash, \flyhash) plus our  method (\densefly). We then describe our methods for providing low-dimensional binning for $h_2$ to \flyhash and \densefly. All algorithms described below are data-independent, meaning that the hash for an input is constructed without using any other input items. Overall, we propose a hybrid fly hashing scheme that takes advantage of high-dimensionality to provide better ranking of candidates, and low-dimensionality to quickly find candidate neighbors to rank. 


\subsection{SimHash}

Charikar~\cite{Charikar2002} proposed the following hashing scheme for generating a binary hash code for an input vector, $x$. First, $mk$ (i.e., the hashing dimension) random projection vectors, $r_1, r_2, \dots, r_{mk}$, are generated, each of dimension $d$. Each element in each random projection vector is drawn uniformly from a Gaussian distribution, $\mathcal{N}(0,1)$. Then, the $i$\textsuperscript{th} value of the binary hash is computed as:
\begin{equation}
h_1(x)_i = 
  \begin{cases}
     1 & \quad\text{if $r_i \cdot x \geq 0$}\\
     0 & \quad\text{if $r_i \cdot x < 0$.}
  \end{cases}
\end{equation}
This scheme preserves distances under the angular and Euclidean distance measures~\cite{Datar2004}.

\subsection{WTAHash (Winner-take-all hash)}

Yagnik et al.~\cite{Yagnik2011} proposed the following binary hashing scheme. First, $m$ permutations, $\theta_1, \theta_2, \dots, \theta_m$ of the input vector are computed. For each permutation $i$, we consider the first $k$ components, and find the index of the component with the maximum value. $C_i$ is then a zero vector of length $k$ with a single 1 at the index of the component with the maximum value. The concatenation of the $m$ vectors, $h_1(x) = [C_1, C_2, \dots, C_{m}]$ corresponds to the hash of $x$. This hash code is sparse --- there is exactly one 1 in each successive block of length $k$ --- and by setting $mk > d$, hashes can be generated that are of dimension greater than the input dimension. We call $k$ the WTA factor.

\wtahash preserves distances under the rank correlation measure~\cite{Yagnik2011}. It also generalizes MinHash~\cite{Broder1997,Shrivastava2014}, and was shown to outperform several data-dependent LSH algorithms, including PCAHash~\cite{Wang2006PCA2,Wang2006PCA1}, spectral hash, and, by transitivity, restricted Boltzmann machines~\cite{Weiss2009}.

\subsection{\flyhash and \densefly}
\label{sec:flyhash}

The two fly hashing schemes (Algorithm~\ref{alg:flyhash}, Figure~\ref{fig:overview}) first project the input vector into an $mk$-dimensional hash space using a sparse, binary random matrix, proven to preserve locality~\cite{Dasgupta2017}. This random projection has a sampling rate of $\alpha$, meaning that in each random projection, only $\lfloor \alpha d\rfloor$ input indices are considered (summed). In the fly circuit, $\alpha\sim 0.1$ since each Kenyon cell samples from roughly 10\% (6/50) of the projection neurons. 

\begin{figure}[h]
\begin{center}
\centerline{\includegraphics[width=\columnwidth]{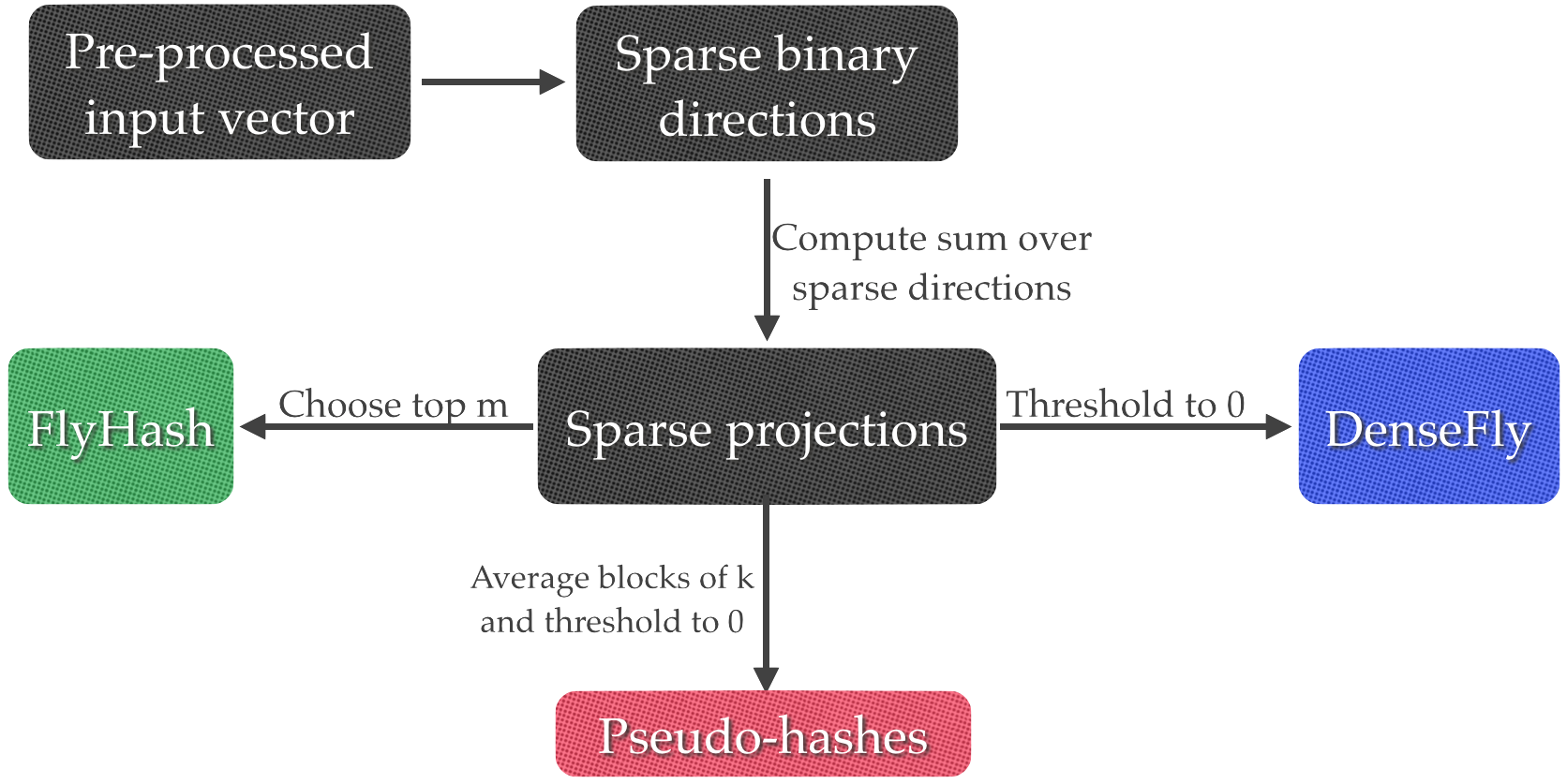}}
\caption{Overview of the fly hashing algorithms.}
\label{fig:overview}
\end{center}
\end{figure}

The first scheme, \flyhash~\cite{Dasgupta2017}, sparsifies and binarizes this representation by setting the indices of the top $m$ elements to 1, and the remaining indices to 0. In the fly circuit, $k=20$, since the top firing 5\% of Kenyon cells are retained, and the rest are silenced by the APL inhibitory neuron. Thus, a \flyhash hash is an $mk$-dimensional vector with exactly $m$ ones, as in \wtahash. However, in contrast to \wtahash, where the WTA is applied locally onto each block of length $k$, for \flyhash, the sparsification happens globally considering all $mk$ indices together. We later prove (Lemma~\ref{lm:po}) that this difference allows more pairwise orders to be encoded within the same hashing dimension.

For \flyhash, the number of unique Hamming distances between two hashes is limited by the hash length $m$. Greater separability can be achieved if the number of $1$s in the high-dimensional hash is allowed to vary. The second scheme, \densefly, sparsifies and binarizes the representation by setting all indices with values $\geq 0$ to 1, and the remaining to 0 (akin to \simhash though in high dimensions). We will show that this method provides even better separability than \flyhash in high dimensions.

\subsection{Multi-probe hashing to find candidate nearest-neighbors}

In practice, the most similar item to a query may have a similar, but not exactly the same, $mk$-dimensional hash as the query. In such a case, it is important to also identify candidate items with a similar hash as the query. Dasgupta et al.~\cite{Dasgupta2017} did not propose a multi-probe binning strategy, without which their \flyhash algorithm is unusable in practice.

\textbf{\simhash.} For low-dimensional hashes, \simhash efficiently probes nearby hash bins using a technique called \emph{multi-probe}~\cite{Lv2007}. All items with the same $mk$-dimensional hash are placed into the same bin; then, given an input $x$, the bin of $h_1(x)$ is probed, as well as all bins within Hamming distance $r$ from this bin. This approach leads to large reductions in search space during retrieval as only bins which differ from the query point by $r$ bits are probed. Notably, even though multi-probe avoids a linear search over all points in the dataset, a linear search over the bins themselves is unavoidable.

\textbf{\flyhash and \densefly.} For high-dimensional hashes, multi-probe is even more essential because even if two input vectors are similar, it is unlikely that their high dimensional hashes will be exactly the same. For example, on the \sift dataset with a WTA factor $k=20$ and $m=16$, \flyhash produces about 860,000 unique hashes (about 86\% the size of the dataset). In contrast, \simhash with $mk=16$ produces about 40,000 unique hashes (about 4\% the size of the dataset). Multi-probing directly in the high-dimensional space using the \simhash scheme, however, is unlikely to reduce the search space without spending significant time probing many nearby bins.

One solution to this problem is to use low-dimensional hashes to reduce the search space and quickly find candidate neighbors, and then to use high-dimensional hashes to rank these neighbors according to their similarity to the query. We introduce a simple algorithm for computing such low-dimensional hashes, called \emph{pseudo-hashes} (Algorithm~\ref{alg:flyhash}). To create an $m$-dimensional pseudo-hash of an $mk$-dimensional hash, we consider each successive block $j$ of length $k$; if the sum (or equivalently, the average) of the activations of this block is $> 0$, we set the $j$\textsuperscript{th} bit of the pseudo-hash to 1, and 0 otherwise. Binning, then, can be performed using the same procedure as \simhash. 

Given a query, we perform multi-probe on its low-dimensional pseudo-hash ($h_2$) to generate candidate nearest neighbors. Candidates are then ranked based on their Hamming distance to the query in high-dimensional hash space ($h_1$). Thus, our approach combines the advantages of low-dimensional probing and high-dimensional ranking of candidate nearest-neighbors. 

\begin{algorithm}[htb]
   \caption{\flyhash and \densefly}
   \label{alg:flyhash}
\begin{algorithmic}
   \STATE {\bfseries Input:} vector $x \in \mathbb{R}^d$, hash length $m$, WTA factor $k$, sampling rate $\alpha$ for the random projection.
   \STATE 
   \STATE \# Generate $mk$ sparse, binary random projections by 
   \STATE \# summing from $\lfloor \alpha d \rfloor$ random indices each.
   \STATE $S = \{S_i\ |\ S_i = \mathrm{rand}(\lfloor \alpha d \rfloor, d)\}$, where $|S| = mk$
   \STATE 
   \STATE \# Compute high-dimensional hash, $h_1$.
   \FOR{$j=1$ {\bfseries to} $mk$}
   \STATE $a(x)_j = \sum_{i \in S_j} x_i$ \quad \# Compute activations
   \ENDFOR
   \STATE
   \IF{\flyhash}
    \STATE $h_1(x) = \mathrm{WTA}(a(x)) \in \{0,1\}^{mk}$ \ \# Winner-take-all
   \ELSIF{\densefly}
    \STATE $h_1(x) = \mathrm{sgn}(a(x)) \in \{0,1\}^{mk}$ \quad \ \# Threshold at 0
   \ENDIF
   \STATE
   \STATE \# Compute low-dimensional pseudo-hash (bin), $h_2$.
   \FOR{$j=1$ {\bfseries to} $m$}
   \STATE {$p(x)_j = \mathrm{sgn}(\sum_{u=k(j-1)+1}^{u=kj} a(x)_u/k)$}
   \ENDFOR
   \STATE $h_2(x) = g(p(x)) \in [0,\dots,b]$ \quad \# Place in bin
   \STATE
\end{algorithmic}
\underline{Note}: The function $\mathrm{rand}(a,b)$ returns a set of $a$ random integers in $[0,b]$. The function $g(\cdot)$ is a conventional hash function used to place a pseudo-hash into a discrete bin.
\end{algorithm}

\textbf{\wtahash.} To our knowledge, there is no method for doing multi-probe with \wtahash. Pseudo-hashing cannot be applied for \wtahash because there is a 1 in every block of length $k$, hence all psuedo-hashes will be a 1-vector of length $m$.

\subsection{Strategy for comparing algorithms}
\label{sec:eval}

We adopt a strategy for fairly comparing two algorithms by equating either their computational cost or their hash dimensionality, as described below.

\textbf{Selecting hyperparameters.} We consider hash lengths $m \in [16,128]$. We compare all algorithms using $k=4$, which was reported to be optimal by Yagnik et al.~\cite{Yagnik2011} for \wtahash, and $k=20$, which is used by the fly circuit (i.e., only the top 5\% of Kenyon cells fire for an odor).

\textbf{Comparing \simhash versus \flyhash.} \simhash random projections are more expensive to compute than \flyhash random projections; this additional expense allows us to compute more random projections (i.e., higher dimensionality) while not increasing the computational cost of generating a hash. Specifically, for an input vector $x$ of dimension $d$, \simhash computes the dot product of $x$ with a dense Gaussian random matrix. Computing the value of each hash dimension requires $2d$ operations: $d$ multiplications plus $d$ additions. \flyhash (effectively) computes the dot product of $x$ with a sparse binary random matrix, with sampling rate $\alpha$. Each dimension requires $\lfloor \alpha d \rfloor$ addition operations only (no multiplications are needed). Using $\alpha=0.1$, as per the fly circuit, to equate the computational cost of both algorithms, the Fly is afforded $k=20$ additional hashing dimensions. Thus, for \simhash, $mk=m$ (i.e., $k=1$) and for \flyhash, $mk = 20m$. The number of ones in the hash for each algorithm may be different. In experiments with $k=4$, we keep $\alpha=0.1$, meaning that both fly-based algorithms have 1/5\textsuperscript{th} the computational complexity as \simhash. 

\textbf{Comparing \wtahash versus \flyhash.} Since \wtahash does not use random projections, it is difficult to equate the computational cost of generating hashes. Instead, to compare \wtahash and \flyhash, we set the hash dimensionality and the number of 1s in each hash to be equal. Specifically, for \wtahash, we compute $m$ permutations of the input, and consider the first $k$ components of each permutation. This produces a hash of dimension $mk$ with exactly $m$ ones. For \flyhash, we compute $mk$ random projections, and set the indices of the top $m$ dimensions to 1. 

\textbf{Comparing \flyhash versus \densefly.} \densefly computes sparse binary random projections akin to \flyhash, but unlikely \flyhash, it does not apply a WTA mechanism but rather uses the sign of the activations to assign a value to the bit, like \simhash. To fairly compare \flyhash and \densefly, we set the hashing dimension ($mk$) to be the same to equate the computational complexity of generating hashes, though the number of ones may differ.


\textbf{Comparing multi-probe hashing.} \simhash uses low-dimensional hashes to both build the hash index and to rank candidates (based on Hamming distances to the query hash) during retrieval. \densefly uses pseudo-hashes of the same low dimensionality as \simhash to create the index; however, unlike \simhash, \densefly uses the high-dimensional hashes to rank candidates. Thus, once the bins and indices are computed, the pseudo-hashes do not need to be stored.  A pseudo-hash for a query is only used to determine which bin to look in to find candidate neighbors.

\subsection{Evaluation datasets and metrics}
\label{sec:metrics}

\textbf{Datasets}. We evaluate each algorithm on six datasets (Table~\ref{tbl:datasets}). There are three datasets with a random subset of 10,000 inputs each (\glove, \labelme, \mnist), and two datasets with 1 million inputs each (\sift, \gist). We also included a dataset of 10,000 random inputs, where each input is a 128-dimensional vector drawn from a uniform random distribution, $\mathcal{U}(0,1)$. This dataset was included because it has no structure and presents a worst-case empirical analysis. For all datasets, the only pre-processing step used is to center each input vector about the mean. 

\begin{table}[h]
\caption{Datasets used in the evaluation.}
\label{tbl:datasets}
\begin{center}
\begin{small}
\begin{tabular}{lrcc}
\toprule
\textbf{Dataset} & \textbf{Size} & \textbf{Dimension} & \textbf{Reference} \\
\midrule
\random   & 10,000        & 128 & --- \\
\glove    & 10,000        & 300 & Pennington et al.~\cite{Pennington2014} \\
\labelme  & 10,000        & 512 & Russell et al.~\cite{Russell2008}    \\
\mnist    & 10,000        & 784 & Lecun et al.~\cite{Lecun1998}      \\
\sift     & 1,000,000     & 128 & Jegou et al.~\cite{Jegou2011}      \\
\gist     & 1,000,000     & 960 & Jegou et al.~\cite{Jegou2011}      \\
\bottomrule
\end{tabular}
\end{small}
\end{center}
\end{table}

\textbf{Accuracy in identifying nearest-neighbors.} Following Yagnik et al.~\cite{Yagnik2011} and Weiss et al.~\cite{Weiss2009}, we evaluate each algorithm's ability to identify nearest neighbors using two performance metrics: area under the precision-recall curve (AUPRC) and mean average precision (mAP). For all datasets, following Jin et al.~\cite{jin2014density}, given a query point, we computed a ranked list of the top 2\% of true nearest neighbors (excluding the query) based on Euclidean distance between vectors in input space. Each hashing algorithm similarly generates a ranked list of predicted nearest neighbors based on Hamming distance between hashes ($h_1$). We then compute the mAP and AUPRC on the two ranked lists. Means and standard deviations are calculated over 500 runs.

\textbf{Time and space complexity.} While mAP and AUPRC evaluate the quality of hashes, in practice, such gains may not be practically usable if constraints such as query time, indexing time, and memory usage are not met. We use two approaches to evaluate the time and space complexity of each algorithm's multi-probe version ($h_2$).

The goal of the first evaluation is to test how the mAP of \simhash and \densefly fare under the same query time. For each algorithm, we hash the query to a bin. Bins nearby the query bin are probed with an increasing search radius. For each radii, the mAP is calculated for the ranked candidates. As the search radius increases, more candidates are pooled and ranked, leading to larger query times and larger mAP scores. 

The goal of the second evaluation is to roughly equate the performance (mAP and query time) of both algorithms and compare the time to build the index and the memory consumed by the index. To do this, we note that to store the hashes, \densefly requires $k$ times more memory to store the high-dimensional hashes. Thus, we allow \simhash to pool candidates from $k$ independent hash tables while using only $1$ hash table for \densefly. While this ensures that both algorithms use roughly the same memory to store hashes, \simhash also requires: (a) $k$ times the computational complexity of \densefly to generate $k$ hash tables, (b) roughly $k$ times more time to index the input vectors to bins for each hash table, and (c) more memory for storing bins and indices. Following Lv et al.~\cite{Lv2007}, we evaluate mAP at a fixed number of nearest neighbors (100). As before, each query is hashed to a bin. If the bin has $\geq 100$ candidates, we stop and rank these candidates. Else, we keep increasing the search radius by 1 until we have least $100$ candidates to rank. We then rank all candidates and compute the mAP versus the true 100 nearest-neighbors. Each algorithm uses the minimal radius required to identify 100 candidates (different search radii may be used by different algorithms).

\section{Results}

First, we present theoretical analysis of the \densefly and \flyhash high-dimensional hashing algorithms, proving that \densefly generates hashes that are locality-sensitive according to Euclidean and cosine distances, and that \flyhash preserves rank similarity for any $\ell_p$ norm; we also prove that pseudo-hashes are effective for reducing the search space of candidate nearest-neighbors without increasing computational complexity. Second, we evaluate how well each algorithm identifies nearest-neighbors using the hash function, $h_1$, based on its query time, computational complexity, memory consumption, and indexing time. Third, we evaluate the multi-probe versions of \simhash, \flyhash, and \densefly ($h_2$). 


\subsection{Theoretical analysis of high-dimensional hashing algorithms}

\begin{lemma}
\densefly generates hashes that are locality-sensitive.
\label{lm:flypartial}
\end{lemma}
\begin{proof}
The idea of the proof is to show that \densefly approximates a high-dimensional \simhash, but at $k$ times lower computational cost. Thus, by transitivity, \densefly preserves cosine and Euclidean distances, as shown for \simhash~\cite{Datar2004}.

The set $S$ (Algorithm 1), containing the indices that each Kenyon cell (KC) samples from, can be represented as a sparse binary matrix, $M$. In Algorithm~1, we fixed each column of $M$ to contain exactly $\lfloor \alpha d \rfloor$ ones. However, maintaining exactly $\lfloor \alpha d \rfloor$ ones is not necessary for the hashing scheme, and in fact, in the fly's olfactory circuit, the number of projection neurons sampled by each KC is approximately a binomial distribution with a mean of 6~\cite{Caron2013,Stevens2015}.
Suppose the projection directions in the fly's hashing schemes (\flyhash and \densefly) are sampled from a binomial distribution; i.e., let $M \in \{0,1\}^{dmk}$ be a sparse binary matrix whose elements are sampled from $dmk$ independent Bernoulli trials each with success probability $\alpha$, so that the total number of successful trials follows $\bd(dmk,\alpha)$. Pseudo-hashes are calculated by averaging $m$ blocks of $k$ sparse projections. Thus, the expected activation of Kenyon cell $j$ to input $x$ is:
\begin{align}
\expt[a_{DenseFly}(x)_j]=\expt[\sum_{u=k(j-1)+1}^{u=kj}\sum_i M_{ui}x_i/k]. \label{eqn:1}
\end{align}
Using the linearity of expectation,
$$\expt[a_{DenseFly}(x)_j]=k\expt[\sum_i M_{ui}x_i]/k,$$
where $u$ is any arbitrary index in $[1,mk]$. Thus, $\expt[a_{DenseFly}(x)_j]=\alpha\sum_{i}x_i$, as $m\rightarrow \infty.$ The expected value of a \densefly activation is given in Equation~\eqref{eqn:1} with special condition that $k=1$.


Similarly, the projection directions in \simhash are sampled from a Gaussian distribution; i.e., let $M^{D} \in \mathbb{R}^{d\times m}$ be a dense matrix whose elements are sampled from $\normal(\mu,\sigma)$. Using linearity of expectation, the expected value of the $j\textsuperscript{th}$ \simhash projection to input $x$ is:
$$\expt[a_{SimHash}(x)_j]=\expt[\sum_i M^{D}_{ji}x_i]=\mu \sum_i x_i.$$ 
Thus, $\expt[a_{DenseFly}(x)_j]=\expt[a_{SimHash}(x)_j]\ \forall\ j\in [1,m]$ if $\mu=\alpha$.

In other words, sparse activations of \densefly approximate the dense activations of \simhash as the hash dimension increases. Thus, a \densefly hash approximates \simhash of dimension $mk$. In practice, this approximation works well even for small values of $m$ since hashes depend only on the sign of the activations. 
\end{proof}

We supported this result by empirical analysis showing that the AUPRC for \densefly is very similar to that of \simhash when using equal dimensions (Figure~S1). \densefly, however, takes $k$-times less computation. In other words, we proved that the computational complexity of \simhash could be reduced $k$-fold while still achieving the same performance.

We next analyze how \flyhash preserves a popular similarity measure for nearest-neighbors, called rank similarity~\cite{Yagnik2011}, and how \flyhash better separates items in high-dimensional space compared to \wtahash (which was designed for rank similarity). Dasgupta et al.~\cite{Dasgupta2017} did not analyze \flyhash for rank similarity neither theoretically nor empirically.


\begin{lemma}
\flyhash preserves rank similarity of inputs under any $\ell_p$ norm.
\end{lemma}
\begin{proof}
The idea is to show that small perturbations to an input vector does not affect its hash.

Consider an input vector $x$ of dimensionality $d$ whose hash of dimension $mk$ is to be computed. The activation of the $j$\textsuperscript{th} component (Kenyon cell) in the hash is given by $a_j=\sum_{i\in S_j}x_i$, where $S_j$ is the set of dimensions of $x$ that the $j$\textsuperscript{th} Kenyon cell samples from. Consider a perturbed version of the input, $x' = x + \delta x$, where $||\delta x||_p=\epsilon$. The activity of the $j$\textsuperscript{th} Kenyon cell to the perturbed vector $x'$ is given by: 
\begin{align}
a'_j=\sum_{i\in S_j}x'_i=a_j+\sum_{i\in S_j}\delta x_i. \nonumber
\end{align}
By the method of Lagrange multipliers, $|a'_j-a_j| \leq d\alpha\epsilon\ /\sqrt[\leftroot{-3}\uproot{3}p]{d\alpha} \ \forall j$ (Supplement).  Moreover, for any index $u\neq j$,
\begin{align}
||a'_j-a'_u|-|a_j-a_u|| & \leq |(a'_j-a'_u)-(a_j-a_u)| \leq 2d\alpha\epsilon\ /\sqrt[\leftroot{-3}\uproot{3}p]{d\alpha}. \nonumber
\end{align}
In particular, let $j$ be the index of $h_1(x)$ corresponding to the smallest activation in the `winner' set of the hash (i.e., the smallest activation such that its bit in the hash is set to 1). Conversely, let $u$ be the index of $h_1(x)$ corresponding to the largest activation in the `loser' set of the hash. Let $\beta=a_j-a_u >0$. Then, $$\beta-2d\alpha\epsilon\ / \sqrt[\leftroot{-3}\uproot{3}p]{d\alpha} \leq |a'_j-a'_u| \leq \beta+2d\alpha\epsilon/\sqrt[\leftroot{-3}\uproot{3}p]{d\alpha}.$$
For $\epsilon < \beta \sqrt[\leftroot{-3}\uproot{3}p]{d\alpha}/2d\alpha$, it follows that $(a'_j-a'_u) \in [\beta - 2d\alpha \epsilon\ / \sqrt[\leftroot{-3}\uproot{3}p]{d\alpha},\beta+2d\alpha \epsilon\ / \sqrt[\leftroot{-3}\uproot{3}p]{d\alpha}]$. 
Thus, $a'_j>a'_u$. Since, $j$ and $u$ correspond to the lowest difference between the elements of the winner and loser sets, it follows that all other pairwise rank orders defined by \flyhash are also maintained. Thus, \flyhash preserves rank similarity between two vectors whose distance in input space is small. As $\epsilon$ increases, the partial order corresponding to the lowest difference in activations is violated first leading to progressively higher Hamming distances between the corresponding hashes.
\end{proof}

\begin{lemma}
\flyhash encodes $m$-times more pairwise orders than \wtahash for the same hash dimension.
\label{lm:po}
\end{lemma}
\begin{proof}
The idea is that \wtahash imposes a local constraint on the winner-take-all (exactly one 1 in each block of length $k$), whereas \flyhash uses a global winner-take-all, which allows \flyhash to encode more pairwise orders. 

We consider the pairwise order function $PO(X,Y)$ defined by Yagnik et al.~\cite{Yagnik2011}, where $(X,Y)$ are the WTA hashes of inputs $(x,y)$. In simple terms, $PO(X,Y)$ is the number of inequalities on which the two hashes $X$ and $Y$ agree.

To compute a hash, \wtahash concatenates pairwise orderings for $m$ independent permutations of length $k$. Let $i$ be the index of the $1$ in a given permutation. Then, $x_i\geq x_j \ \forall \ j \in [1,k]\setminus \{i\}$. Thus, a \wtahash denotes $m(k-1)$ pairwise orderings. The WTA mechanism of \flyhash encodes pairwise orderings for the top $m$ elements of the activations, $a$. Let $W$ be the set of the top $m$ elements of $a$ as defined in Algorithm~\ref{alg:flyhash}. Then, for any $j \in W$, $a_j\geq a_i\ \forall i\in [1,mk]\setminus W$. Thus, each $j \in W$ denotes $m(k-1)$ inequalities, and \flyhash encodes $m^2(k-1)$ pairwise orderings. Thus, the pairwise order function for \flyhash encodes $m$ times more orders.
\end{proof}

Empirically, we found that \flyhash and \densefly achieved much higher Kendall-$\tau$ rank correlation than \wtahash, which was specifically designed to preserve rank similarity~\cite{Yagnik2011} (Results, Table~\ref{tbl:kt}). This validates our theoretical results. 

\begin{lemma} 
Pseudo-hashes approximate \simhash with increasing WTA factor $k$.
\end{lemma}
\vspace{-0.15in}
\begin{proof}
The idea is that expected activations of pseudo-hashes calculated from sparse projections is the same as the activations of \simhash calculated from dense projections. 

The analysis of Equation~\eqref{eqn:1} can be extended to show that pseudo-hashes approximate \simhash of the same dimensionality. Specifically,
%

%
$$\expt[a_{pseudo}(x)_j]=\alpha\sum_{i}x_i,\ as\ k\rightarrow \infty.$$
%
Similarly, the projection directions in \simhash are sampled from a Gaussian distribution; i.e., let $M^{D} \in \mathbb{R}^{d\times m}$ be a dense matrix whose elements are sampled from $\normal(\mu,\sigma)$. Using linearity of expectation, the expected value of the $j\textsuperscript{th}$ \simhash projection is:
$$\expt[a_{SimHash}(x)_j]=\expt[\sum_i M^{D}_{ji}x_i]=\mu \sum_i x_i.$$ 
Thus, $\expt[a_{pseudo}(x)_j]=\expt[a_{SimHash}(x)_j]\ \forall\ j\in [1,m]$ if $\mu=\alpha$.Similarly, the variances of $a_{SimHash}(x)$ and $a_{pseudo}(x)$ are equal if $\sigma^2=\alpha(1-\alpha)$. Thus, \simhash itself can be interpreted as the pseudo-hash of a \flyhash with very large dimensions. 
\end{proof}
Although in theory, this approximation holds for only large values of $k$, in practice the approximation can operate under a high degree of error since equality of hashes requires only that the sign of the activations of pseudo-hash be the same as that of \simhash. 

Empirically, we found that the performance of only using pseudo-hashes (not using the high-dimensional hashes) for ranking nearest-neighbors performs similarly with \simhash for values of $k$ as low as $k=4$ (Figure~S2 and S3), confirming our theoretical results. Notably, the computation of pseudo-hashes is performed by re-using the activations for \densefly, as explained in Algorithm \ref{alg:flyhash} and Figure \ref{fig:overview}. Thus, pseudo-hashes incur little computational cost and provide an effective tool for reducing the search space due to their low dimensionality.


\subsection{Empirical evaluation of low- versus high-dimensional hashing}

\begin{figure*}[bt]
\begin{center}
\includegraphics[width=\textwidth]{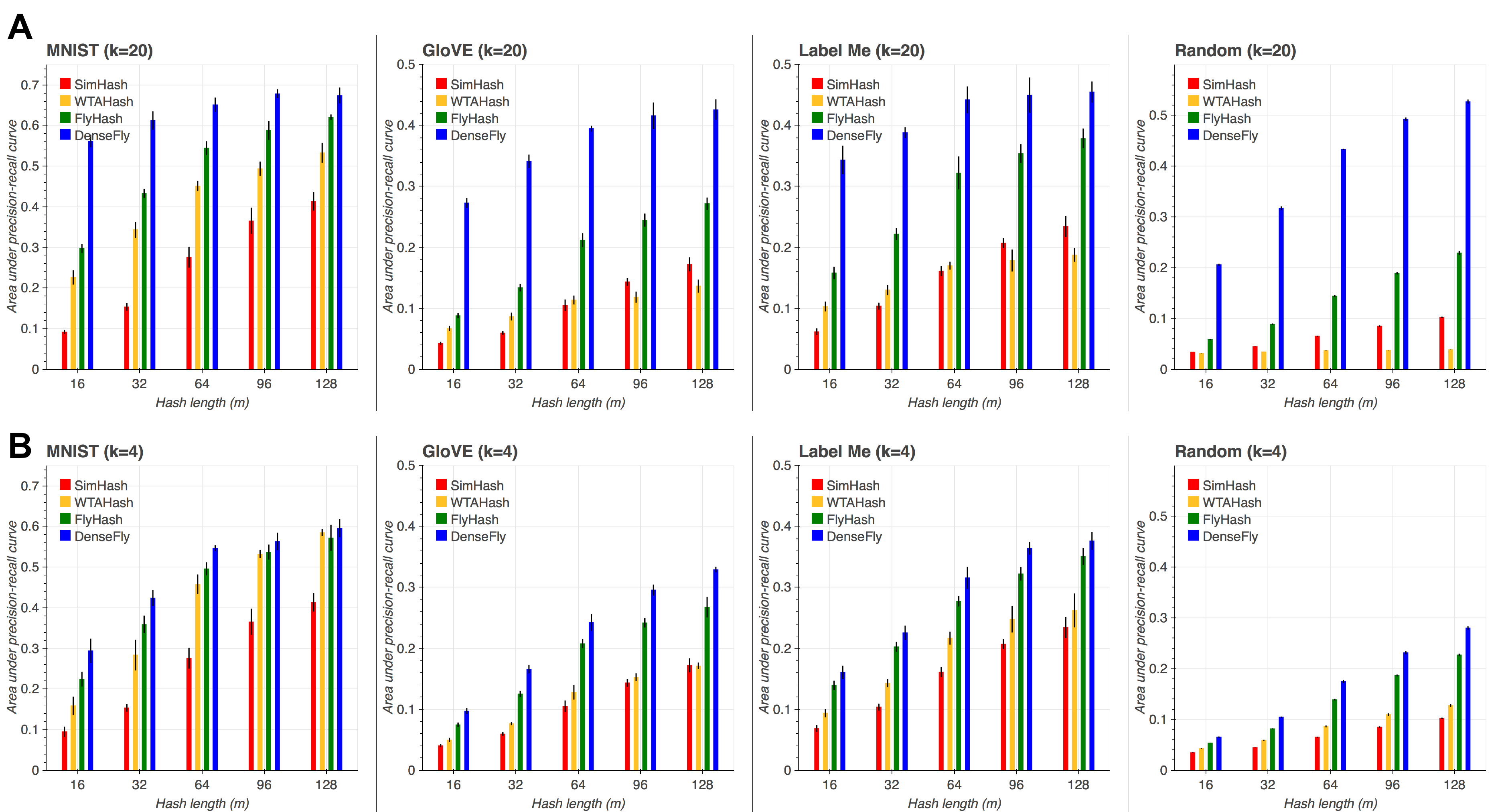}
\caption{\textbf{Precision-recall for the \mnist, \glove, \labelme, and \random datasets}. A) $k=20$. B) $k=4$. In each panel, the $x$-axis is the hash length, and the $y$-axis is the area under the precision-recall curve (higher is better). For all datasets and hash lengths, \densefly performs the best.}
\label{fig:10kprc}
\end{center}
\end{figure*}

We compared the quality of the hashes ($h_1$) for identifying the nearest-neighbors of a query using the four 10k-item datasets (Figure~\ref{fig:10kprc}A). For nearly all hash lengths, \densefly outperforms all other methods in area under the precision-recall curve (AUPRC). For example, on the \glove dataset with hash length $m=64$ and WTA factor $k=20$, the AUPRC of \densefly is about three-fold higher than \simhash and \wtahash, and almost two-fold higher than \flyhash (DenseFly=0.395, FlyHash=0.212, SimHash=0.106, WTAHash=0.112). On the \random dataset, which has no inherent structure, \densefly provides a higher degree of separability in hash space compared to \flyhash and \wtahash, especially for large $k$ (e.g., nearly 0.440 AUPRC for \densefly versus 0.140 for \flyhash, 0.037 for \wtahash, and 0.066 for \simhash with $k=20, m=64$). Figure~\ref{fig:10kprc}B shows empirical performance for all methods using $k=4$, which shows similar results.

\densefly also outperforms the other algorithms in identifying nearest neighbors on two larger datasets with 1M items each (Figure~\ref{fig:1mprc}). For example, on \sift with $m=64$ and $k=20$, \densefly achieves 2.6x/2.2x/1.3x higher AUPRC compared to \simhash, \wtahash, and \flyhash, respectively. These results demonstrate the promise of high-dimensional hashing on practical datasets.

\begin{figure*}[bt]
\begin{center}
\includegraphics[width=\textwidth]{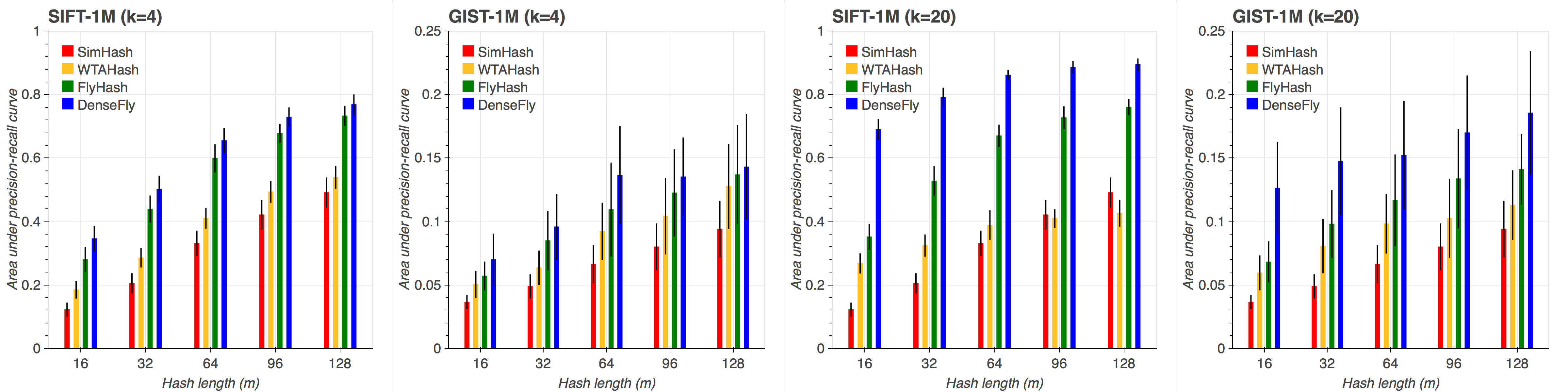}
\caption{\textbf{Precision-recall for the \sift and \gist datasets}. In each panel, the $x$-axis is the hash length, and the $y$-axis is the area under the precision-recall curve (higher is better). The first two panels shows results for \sift and \gist using $k=$; the latter two show results for $k=20$. \densefly is comparable to or outperforms all other algorithms.}
\label{fig:1mprc}
\end{center}
\end{figure*}


\subsection{Evaluating multi-probe hashing}

Here, we evaluated the multi-probing schemes of \simhash and \densefly (pseudo-hashes). Using $k=20$, \densefly achieves higher mAP for the same query time (Figure~\ref{fig:multi}A). For example, on the \glove dataset, with a query time of 0.01 seconds, the mAP of \densefly is 91.40\% higher than that of \simhash, with similar gains across other datasets. Thus, the high-dimensional \densefly is better able to rank the candidates than low-dimensional \simhash. Figure~\ref{fig:multi}B shows that similar results hold for $k=4$; i.e., \densefly achieves higher mAP for the same query time as \simhash. 


\begin{figure*}[tbp]
\begin{center}
\includegraphics[width=\textwidth]{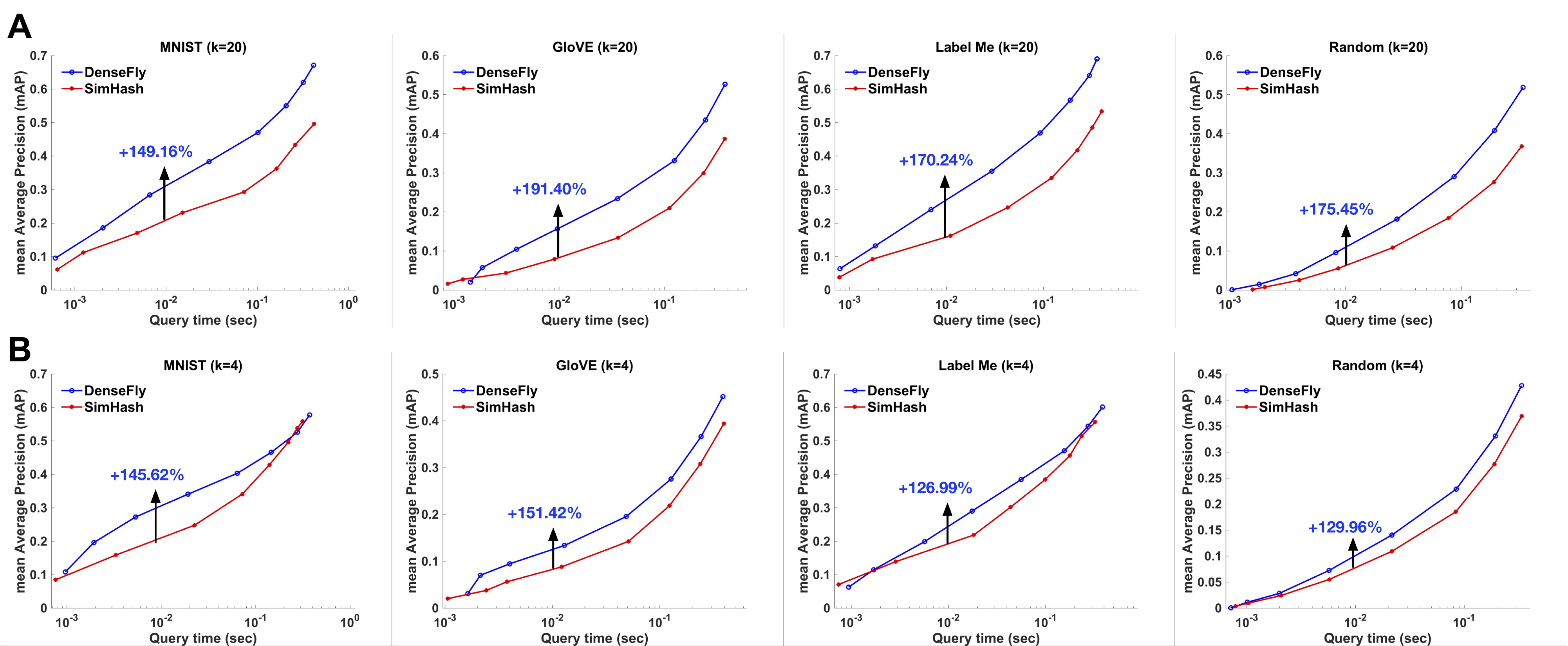}
\caption{\textbf{Query time versus mAP for the 10k-item datasets.} A) $k=20$. B) $k=4$.  In each panel, the $x$-axis is query time, and the $y$-axis is mean average precision (higher is better) of ranked candidates using a hash length $m=16$. Each successive dot on each curve corresponds to an increasing search radius. For nearly all datasets and query times, \densefly with pseudo-hash binning performs better than \simhash with multi-probe binning. The arrow in each panel indicates the gain in performance for \densefly at a query time of 0.01 seconds.} 
\label{fig:multi}
\end{center}
\vskip -0.2in
\end{figure*}

Next, we evaluated the multi-probe schemes of \simhash, \flyhash (as originally conceived by Dasgupta et al.~\cite{Dasgupta2017} without multi-probe), our \flyhash multi-probe version (called FlyHash-MP), and \densefly based on mAP as well as query time, indexing time, and memory usage. To boost the performance of \simhash, we pooled and ranked candidates over $k$ independent hash tables as opposed to $1$ table for \densefly (Section~\ref{sec:metrics}). Table~\ref{tbl:probe} shows that for nearly the same mAP as \simhash, \densefly significantly reduces query times, indexing times, and memory consumption. For example, on the Glove-10K dataset, \densefly achieves marginally lower mAP compared to \simhash (0.966 vs.\@ 1.000) but requires only a fraction of the querying time (0.397 vs.\@ 1.000), indexing time (0.239 vs.\@ 1.000) and memory (0.381 vs.\@ 1.000). Our multi-probe \flyhash algorithm improves over the original \flyhash, but it still produces lower mAP compared to \densefly. Thus, \densefly more efficiently identifies a small set of high quality candidate nearest-neighbors for a query compared to the other algorithms.



\begin{table*}[tbp]
\caption{\textbf{Performance of multi-probe hashing for four datasets.} Across all datasets, \densefly achieves similar mAP as \simhash, but with 2x faster query times, 4x fewer hash tables, 4--5x less indexing time, and 2--4x less memory usage. FlyHash-MP evaluates our multi-probe technique applied to the original FlyHash algorithm. \densefly and FlyHash-MP require similar indexing time and memory, but \densefly achieves higher mAP. \flyhash without multi-probe ranks the entire database per query; it therefore does not build an index and has large query times. Performance is shown normalized to that of \simhash. We used WTA factor, $k=4$ and hash length, $m=16$. } 
\label{tbl:probe}
\begin{center}
\begin{small}
\begin{tabular}{llcccrrr}
\toprule
\textbf{Dataset} & \textbf{Algorithm} & \textbf{\# Tables} & \textbf{mAP @ 100} & \textbf{Query} & \textbf{Indexing} & \textbf{Memory} \\
\midrule
%
GIST-100k    & \simhash    & 4 & 1.000 &  1.000 & 1.000 & 1.000 \\
         & \densefly   & 1 & 0.947 &  0.537 & 0.251 & 0.367 \\
         & \flyhashmp  & 1 & 0.716 &  0.515 & 0.252 & 0.367 \\
         & \flyhash    & 1 & 0.858 &  5.744 & 0.000 & 0.156 \\ \midrule
%
%
\mnist-10k & \simhash    & 4 & 1.000 &  1.000 & 1.000 & 1.000 \\
           & \densefly   & 1 & 0.996 &  0.669 & 0.226 & 0.381 \\   
           & \flyhashmp  & 1 & 0.909 &  0.465 & 0.232 & 0.381 \\      
           & \flyhash    & 1 & 0.985 &  1.697 & 0.000 & 0.174 \\ \midrule
%
%
\labelme-10k  & \simhash  & 4 & 1.000 &  1.000 & 1.000 & 1.000 \\
          & \densefly   & 1 & 1.075 &  0.481 & 0.242 & 0.383 \\
          & \flyhashmp  & 1 & 0.869 &  0.558 & 0.250 & 0.383 \\
          & \flyhash    & 1 & 0.868 &  2.934 & 0.000 & 0.177 \\ \midrule
%
\glove-10k & \simhash    & 4 & 1.000 &  1.000 & 1.000 & 1.000 \\
           & \densefly   & 1 & 0.966 &  0.397 & 0.239 & 0.381 \\
           & \flyhashmp  & 1 & 0.950 &  0.558 & 0.241 & 0.381 \\
           & \flyhash    & 1 & 0.905 &  1.639 & 0.000 & 0.174 \\ \midrule  
%
\bottomrule
\end{tabular}
\end{small}
\end{center}
\vspace{-0.1in}
\end{table*}

\subsection{Empirical analysis of rank correlation for each method}

Finally, we empirically compared \densefly, \flyhash, and \wtahash based on how well they preserved rank similarity~\cite{Yagnik2011}. For each query, we calculated the $\ell_2$ distances of the top 2\% of true nearest neighbors. We also calculated the Hamming distances between the query and the true nearest-neighbors in hash space. We then calculated the Kendell-$\tau$ rank correlation between these two lists of distances. Across all datasets and hash lengths tested, \densefly outperformed both \flyhash and \wtahash (Table~\ref{tbl:kt}), confirming our theoretical results. 

\begin{table*}[h]
\caption{\textbf{Kendall-$\mathbf{\tau}$ rank correlations for all 10k-item datasets}.  Across all datasets and hash lengths, \densefly achieves a higher rank correlation between $\ell_2$ distance in input space and $\ell_1$ distance in hash space. Averages and standard deviations are shown over 100 queries. All results shown are for WTA factor, $k=20$. Similar performance gains for \densefly over other algorithms with $k=4$ (not shown).} 
\label{tbl:kt}
\begin{center}
\begin{small}
\begin{tabular}{lc|ccc}
\toprule
\textbf{Dataset} & \textbf{Hash Length} & \textbf{\wtahash} & \textbf{\flyhash} & \textbf{\densefly} \\ \midrule
\mnist   & 16 & $0.204 \pm 0.10$ & $0.288 \pm 0.10$ & $\mathbf{0.425 \pm 0.08}$ \\
         & 32 & $0.276 \pm 0.10$ & $0.375 \pm 0.10$ & $\mathbf{0.480 \pm 0.13}$ \\ 
         & 64 & $0.333 \pm 0.10$ & $0.446 \pm 0.11$ & $\mathbf{0.539 \pm 0.12}$ \\ \midrule
\glove   & 16 & $0.157 \pm 0.10$ & $0.189 \pm 0.10$ & $\mathbf{0.281 \pm 0.11}$ \\ 
         & 32 & $0.169 \pm 0.10$ & $0.224 \pm 0.11$ & $\mathbf{0.306 \pm 0.11}$ \\
         & 64 & $0.183 \pm 0.11$ & $0.243 \pm 0.13$ & $\mathbf{0.311 \pm 0.13}$ \\ \midrule
\labelme & 16 & $0.141 \pm 0.08$ & $0.174 \pm 0.08$ & $\mathbf{0.282 \pm 0.08}$ \\
       & 32 & $0.157 \pm 0.08$ & $0.227 \pm 0.09$ & $\mathbf{0.342 \pm 0.11}$ \\
         & 64 & $0.191 \pm 0.09$ & $0.292 \pm 0.10$ & $\mathbf{0.368 \pm 0.10}$ \\ \midrule
\random  & 16 & $0.037 \pm 0.06$ & $0.089 \pm 0.05$ & $\mathbf{0.184 \pm 0.04}$ \\
     & 32 & $0.043 \pm 0.05$ & $0.120 \pm 0.05$ & $\mathbf{0.226 \pm 0.05}$ \\
         & 64 & $0.051 \pm 0.04$ & $0.155 \pm 0.05$ & $\mathbf{0.290 \pm 0.04}$ \\ \midrule
\bottomrule
\end{tabular}
\end{small}
\end{center}
\end{table*}

\section{Conclusions}

%

We analyzed and evaluated a new family of neural-inspired binary locality-sensitive hash functions that perform better than existing data-independent methods (\simhash, \wtahash, \flyhash) across several datasets and evaluation metrics. The key insight is to use efficient projections to generate high-dimensional hashes, which we showed can be done without increasing computation or space complexity. We proved theoretically that \densefly is locality-sensitive under the Euclidean and cosine distances, and that \flyhash preserves rank similarity for any $\ell_p$ norm. We also proposed a multi-probe version of our algorithm that offers an efficient binning strategy for high-dimensional hashes, which is important for making this scheme usable in practical applications. Our method also performs well with only 1 hash table, which also makes this approach easier to deploy in practice. Overall, our results support findings that dimensionality expansion may be a ``blessing''~\cite{Gorban2018,Delalleau2011,Chen2013}, especially for promoting separability for nearest-neighbors search. 


There are many directions for future work. First, we focused on data-independent algorithms; biologically, the fly can ``learn to hash''~\cite{Hige2015} but learning occurs online using reinforcement signals, as opposed to offline from a fixed database~\cite{Wang2016}. Second, we fixed the sampling rate $\alpha=0.10$, as per the fly circuit; however, more work is needed to understand how optimal sampling complexity changes with respect to input statistics and noise. Third, most prior work on multi-probe LSH have assumed that hashes are low-dimensional; while pseudo-hashes represent one approach for binning high-dimensional data via a low-dimensional intermediary, more work is needed to explore other possible strategies. Fourth, there are methods to speed-up random projection calculations, for both Gaussian matrices~\cite{Dasgupta2011,Andoni2015} and sparse binary matrices, which can be applied in practice. 

\section{Code availability}
Source code for all algorithms is available at: http://www.github.com/dataplayer12/Fly-LSH

\ifCLASSOPTIONcaptionsoff
  \newpage
\fi



\bibliographystyle{IEEEtran}
\bibliography{IEEEabrv,references}
%



%

\begin{IEEEbiographynophoto}{Jaiyam Sharma} received his Bachelor of Technology in Engineering Physics from Indian Institute of Technology Delhi, India. He received a Master of Engineering in Electrical and Electronic Engineering from Toyohashi University of Technology, Japan. His research interests are computer vision algorithms for medical diagnostics. He is currently a doctoral candidate at The University of Electro-Communications, Tokyo and a collaborator with Prof. Saket Navlakha at the Salk Institute.
\end{IEEEbiographynophoto}

\begin{IEEEbiographynophoto}{Saket Navlakha} is an assistant professor in the Integrative Biology Laboratory at the Salk Institute for Biological Studies. He received an A.A.\@ from Simon's Rock College in 2002, a B.S.\@ from Cornell University in 2005, and a Ph.D.\@ in computer science from the University of Maryland College Park in 2010. He was then a post-doc in the Machine Learning Department at Carnegie Mellon University until 2014. His research interests include designing algorithms to study the structure and function of biological networks, and the study of ``algorithms in nature''.

\end{IEEEbiographynophoto}






\twocolumn[
\begin{center}

\maketitle{\huge\sffamily Improving Similarity Search with High-dimensional Locality-sensitive Hashing\vspace{2ex}\\Supplementary Information}   
\end{center}
\begin{center}
\vspace{2ex}
\maketitle{Jaiyam Sharma, Saket Navlakha}
\end{center}
]

\section{Empirical analysis of Lemma 1}

To support our theoretical anlaysis of Lemma 1 (main text), we performed an empirical analysis showing that the AUPRC for \densefly is very similar to that of \simhash when using equal dimensions (Figure~\ref{fig:s1}). \densefly, however, takes $k$-times less computation. In other words, we proved that the computational complexity of \simhash could be reduced $k$-fold while still achieving the same performance.

\begin{figure*}[bt]
\begin{center}
\centerline{\includegraphics[width=\textwidth]{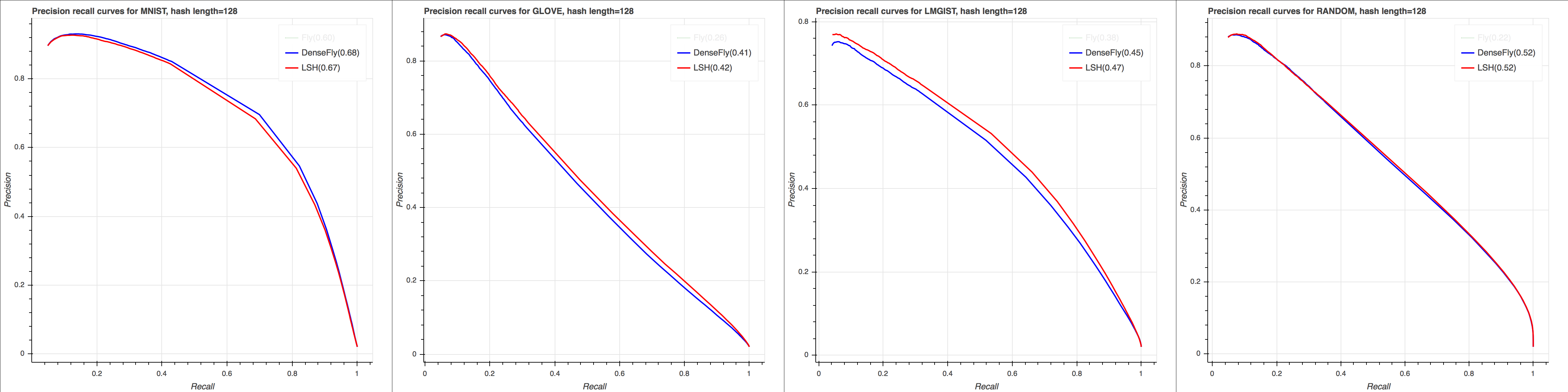}}
\caption{Empirical evaluation of \densefly with \simhash when using the same hash dimension. These empirical results support Lemma 1. Results are shown for $k=20$.}
\label{fig:s1}
\end{center}
\end{figure*}

\section{Comparing Pseudo-hashes with SimHash (Lemma 4)}


Lemma 4 proved that pseudo-hashes approximate SimHash with increasing WTA factor, $k$. Empirically, we found that the performance of only using pseudo-hashes (not using the high-dimensional hashes) for ranking nearest-neighbors performs similarly with \simhash for values of $k$ as low as $k=4$ (Figure~\ref{fig:pseudomap} and Figure~\ref{fig:pseudorecall}), confirming our theoretical results.

\begin{figure*}[htbp]
\begin{center}
\centerline{\includegraphics[width=\textwidth]{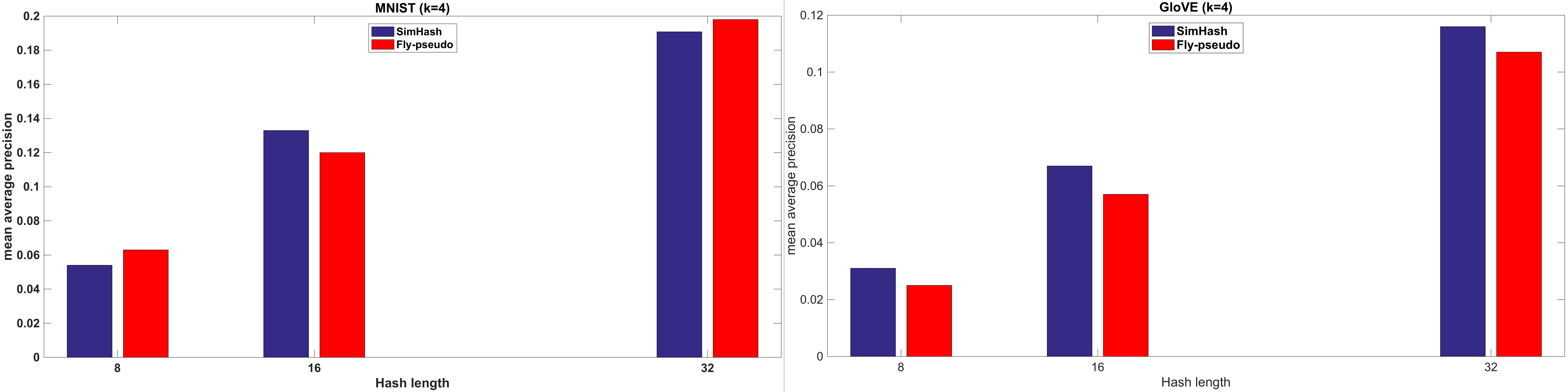}}
\caption{Mean average precision of pseudo-hashes for $k=4$ on the \mnist and \glove datasets. The mAP scores were calculated over 500 queries. \densefly pseudo-hashes and \simhash perform similarly.}
\label{fig:pseudomap}
\end{center}
\vskip -0.1in
\end{figure*}

\begin{figure*}[htbp]
\begin{center}
\centerline{\includegraphics[width=\textwidth]{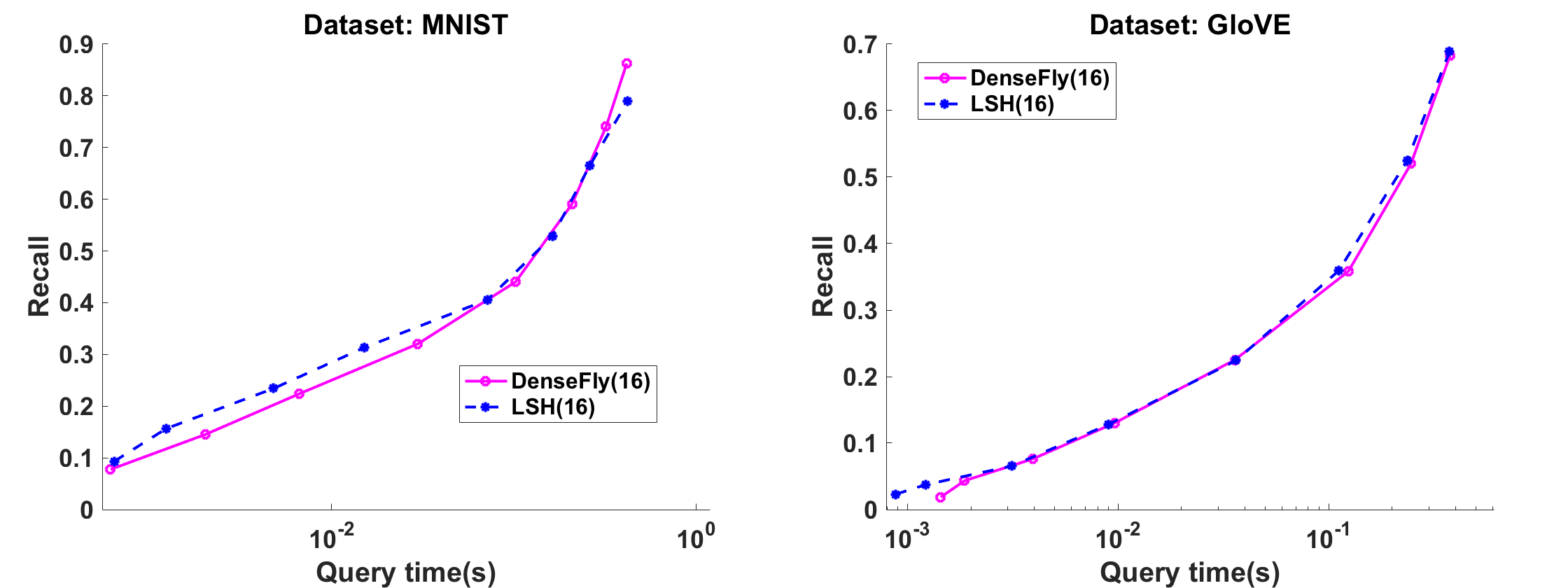}}
\caption{Recall versus query time for \mnist and \glove by \simhash (LSH) and \densefly. Ranking of candidates for \densefly is done by pseudo-hashes. The recall for pseudo-hashes is nearly the same as \simhash across all search radii (query times). This supports the theoretical argument that pseudo-hashes approximate \simhash.} 
\label{fig:pseudorecall}
\end{center}
\vskip -0.1in
\end{figure*}

\section{Bounds of $|a'_j-a_j|$}

Here we derive a result used in the proof of Lemma 2 (main text). Let $[a,b]$ denote the set of all integers from $a$ to $b$. 

Consider $|a_j'-a_j|$ as defined in the proof of Lemma 2. Since $a_j'=\sum_{i\in S_j}x'_i=a_j+\sum_{i\in S_j}\delta x_i$, then $|a_j'-a_j|=|\sum_{i\in S_j}\delta x_i|$. The problem of finding the maximum value of $|a'_j-a_j|$ is one of constrained optimization and can be solved, generally speaking, by using the Karush-Kuhn-Tucker conditions. In this case the solution can also be found using Lagrange multipliers as we show below.

Let $h(\{\delta x_i| i\in S_j\})\equiv|\sum_{i\in S_j} \delta x_i|$. The problem is to maximize $h(\{\delta x_i| i\in S_j\})$ such that $\sum_{t=1}^{mk} |\delta x_t|^p = \epsilon^p$. 

This translates to an equivalent condition $\sum_{i\in S_j} |\delta x_i|^p \leq \epsilon^p$. Since $|S_j|=\lfloor d\alpha \rfloor$, we reformulate $h$ without loss of generality as $h(\delta x_1,\delta x_2,\mathellipsis\ \delta x_{d\alpha})= |\sum_{i=1}^{i=d\alpha} \delta x_i|$, where we drop $\lfloor . \rfloor$ notation for simplicity. Also, we note that $|\sum_{i=1}^{i=d\alpha} \delta x_i| \leq \sum_{i=1}^{i=d\alpha} |\delta x_i|$, where the equality holds  \textit{if and only if} $\delta x_i\geq 0 \ \forall i\in [1,d\alpha]$. Thus, the absolute value signs can be dropped (if the solution found by dropping the absolute value is indeed $\geq 0$). 

Next, let
$f(\delta x_1,\mathellipsis,\delta x_{d\alpha})\equiv\sum_{i=1}^{i=d\alpha} \delta x_i$, subject to the constraint $g(\delta x_1,..\delta x_{d\alpha}) \leq 0$ where $g(\delta x_1,\mathellipsis\ \delta x_{d\alpha})=\sum_{i=1}^{i=d\alpha}\delta x_i^p-\epsilon^p$. We note that the global maximum of $f$ lies outside the $\epsilon$-ball defined by $g$. Thus, the constraint $g$ is active at the optimal solution so that $g(\delta x_1,\mathellipsis\ \delta x_{d\alpha})=0$. Thus, the optimal solution is calculated using the Lagrangian:
\begin{align}
\lagr(\delta x_1,..,\delta x_{d\alpha},\lambda) & = \sum_{i=1}^{i=d\alpha} \delta x_i-\lambda\ (\sum_{i=1}^{i=d\alpha}\delta x_i^p-\epsilon^p).\notag\\
\frac{\partial \lagr}{\partial \delta x_i} & = 1-p\lambda \delta x_i^{p-1} \  \forall i\in [1,d\alpha] \notag\\
\frac{\partial \lagr}{\partial\ \lambda} & = \sum_{i=1}^{i=d\alpha}\delta x_i^p-\epsilon^p \notag
\end{align}
Setting $\frac{\partial \lagr}{\partial \delta x_i}=0$, we get $\delta x_i=(\frac{1}{p\lambda})^{1/(p-1)} \equiv \gamma \forall\ i\in [1,d\alpha].$ Setting $\frac{\partial \lagr}{\partial\ \lambda}=0$, we get $d\alpha\gamma^p=\epsilon^p$. 

Thus, $\gamma=\epsilon/\sqrt[\leftroot{-3}\uproot{3}p]{d\alpha}$ is the only admissible solution for any $p$ since $\delta x_i\geq 0\ \forall\ i\in [1,d\alpha]$ and $\gamma >0$. Therefore, $f(\delta x_1,\mathellipsis,\delta x_{d\alpha})\leq d\alpha\epsilon /\sqrt[\leftroot{-3}\uproot{3}p]{d\alpha}$ and the proof follows.

\ifCLASSOPTIONcaptionsoff
  \newpage
\fi

\end{document}